\documentclass{article} 
\usepackage{iclr2027_conference,times}
\iclrfinalcopy

\usepackage{amsmath,amsfonts,bm}

\def\eqref#1{equation~\ref{#1}}

\def\1{\bm{1}}

\def\eps{{\epsilon}}

\def\vtheta{{\bm{\theta}}}

\def\vb{{\bm{b}}}

\def\vn{{\bm{n}}}

\def\vs{{\bm{s}}}

\def\vx{{\bm{x}}}

\def\vz{{\bm{z}}}

\DeclareMathAlphabet{\mathsfit}{\encodingdefault}{\sfdefault}{m}{sl}
\SetMathAlphabet{\mathsfit}{bold}{\encodingdefault}{\sfdefault}{bx}{n}

\newcommand{\E}{\mathbb{E}}

\newcommand{\R}{\mathbb{R}}

\usepackage{hyperref}
\usepackage{url}

\setcitestyle{numbers,square,comma}

\usepackage[utf8]{inputenc}
\usepackage[T1]{fontenc}
\usepackage{amsmath}
\usepackage{amsfonts}
\usepackage{amssymb}
\usepackage{amsthm}
\usepackage{mathtools}
\usepackage{bm}
\usepackage{booktabs}
\usepackage{array}
\usepackage{multirow}
\usepackage{makecell}
\usepackage{adjustbox}
\usepackage{nicefrac}
\usepackage{microtype}
\usepackage{url}
\usepackage{graphicx}
\usepackage{subcaption}
\usepackage[table]{xcolor}
\usepackage{algorithm}
\usepackage{algpseudocode}
\usepackage{pgfplots}
\usepackage{caption}
\usepackage{subcaption}

\definecolor{darkblue}{rgb}{0.0, 0.0, 0.55}

\usepackage{hyperref}
\hypersetup{
    colorlinks=true,
    linkcolor=darkblue,
    citecolor=darkblue,
    urlcolor=darkblue,
}

\definecolor{MedCoreBlue}{RGB}{225,238,250}
\definecolor{MedCoreGreen}{RGB}{226,245,228}

\theoremstyle{plain}
\newtheorem{theorem}{Theorem}[section]
\newtheorem{proposition}[theorem]{Proposition}

\theoremstyle{definition}
\newtheorem{definition}[theorem]{Definition}
\newtheorem{assumption}[theorem]{Assumption}
\theoremstyle{remark}

\newcommand{\cB}{\mathcal{B}}

\newcommand{\cD}{\mathcal{D}}
\newcommand{\cE}{\mathcal{E}}
\newcommand{\cG}{\mathcal{G}}
\newcommand{\cH}{\mathcal{H}}
\newcommand{\cL}{\mathcal{L}}
\newcommand{\cM}{\mathcal{M}}
\newcommand{\cR}{\mathcal{R}}

\newcommand{\vphi}{\bm{\phi}}

\title{MedCore: Boundary-Preserving Medical Core Pruning for MedSAM}

\author{Cenwei Zhang, Suncheng Xiang\thanks{Corresponding authors.} \\
Shanghai Jiao Tong University\\
Shanghai, China \\
\texttt{cwzhang2001@gmail.com} \& \texttt{xiangsuncheng17@sjtu.edu.cn}\\
\And
Lei You\footnotemark[1] \\
Technical University of Denmark \\
Ballerup, Denmark \\
\texttt{leiyo@dtu.dk}
}

\begin{document}

\maketitle

\begin{abstract}

Medical segmentation foundation models such as SAM and MedSAM provide strong prompt-driven segmentation, but their image encoders are still too large for many clinical settings. Compression is also risky in medicine because a model can keep high Dice while losing boundary fidelity. We propose \textbf{MedCore}, a structured pruning framework for MedSAM. The main idea is to preserve two kinds of structures: structures that became important during SAM-to-MedSAM adaptation, and structures that have high boundary leverage. We identify the first type by a dual-intervention score comparing zeroing a group with resetting it to its original SAM weights. We identify the second type by boundary-aware Fisher estimation. We also introduce a boundary leverage principle, which relates first-order boundary displacement to boundary logit perturbation divided by the spatial gradient norm. This principle explains why boundary metrics can degrade even when Dice remains high. On polyp segmentation benchmarks, MedCore reduces parameters by 60.0\% and FLOPs by 58.4\% while achieving Dice 0.9549, Boundary F1 0.6388, and HD95 5.14 after post-pruning fine-tuning. It also reaches \textbf{86.6\%} parameter reduction and \textbf{90.4G} FLOPs, retaining Dice and HD95 near the fine-tuned MedSAM reference, with lower BF1. Our analysis shows that MedSAM lies in a head-fragile boundary regime: the median 95th-percentile boundary leverage is 2.887 times as high for head-pruning as for MLP-pruning steps, consistent with BF1 and HD95 degradation. Our code is available at \url{https://github.com/cenweizhang/MedCore}.

\end{abstract}

\section{Introduction}
\label{intro}

Segment Anything Model (SAM) changed image segmentation by turning segmentation into a prompt-driven task \citep{kirillov2023segment}. MedSAM adapts this idea to medical images and gives strong results across many modalities \citep{ma2024segment}. These models are useful because they can accept prompts such as bounding boxes and produce high-quality masks. Their cost is still high. The MedSAM image encoder is a ViT-based model~\cite{dosovitskiy2021vit}, and its inference cost can be difficult to support in real-time endoscopy, point-of-care ultrasound, and other resource-limited settings.

A natural solution is model compression. Yet medical segmentation has a failure mode that ordinary compression metrics can hide. A compressed model may still identify the coarse lesion region and keep a high Dice score, while its predicted boundary becomes thick, broken, or shifted. This is a serious issue because boundary quality affects lesion measurement, treatment planning, and clinical interpretation. This paper studies structured pruning for MedSAM under a boundary-sensitive setting, directly pruning the existing adapted model rather than training a lightweight surrogate. 
Since MedSAM is adapted from SAM instead of trained from scratch, we exploit the transition from the SAM checkpoint $\vtheta^S$ to the MedSAM checkpoint $\vtheta^M$ to identify which structures remain general-purpose and which become medical-specific.

We propose \textbf{MedCore}, a boundary-preserving medical core pruning framework that keeps a compact yet medically adapted MedSAM subnetwork. For each structure
group $g$, MedCore uses a dual-intervention score that combines its current prediction contribution under zeroing and its medical-adaptation contribution under reset to the original SAM weights, which are two counterfactual questions. A group is safe to prune only when both signals suggest low importance.

We further introduce a boundary leverage principle by viewing the predicted mask boundary as a level set of the logit map. If pruning changes the logit map by $\delta_G(\vx)$, the first-order boundary displacement is proportional to $\delta_G(\vx)$ and inversely proportional to the spatial gradient norm $\|\nabla s_{\vtheta}(\vx)\|_2$. This gives a direct geometric explanation of boundary collapse. Compression is dangerous when it creates a large logit perturbation on a boundary where the logit slope is small. We use this principle to motivate boundary-aware Fisher scoring and to analyze the different roles of attention heads and MLP connections.

In summary, Our three contributions are as follows:
\begin{itemize}
    \item We formulate MedSAM pruning as medical-core extraction and introduce a dual-intervention score based on zeroing and reset-to-SAM operations.
    \item We propose boundary-aware Fisher and Cross-Fisher approximations to make the intervention score practical and boundary-sensitive.
    \item We introduce boundary leverage as a simple geometric principle for explaining Dice-boundary mismatch and head-MLP pruning trade-offs.
\end{itemize}

Our results further show that this distinction matters. Experiments on polyp endoscopy, breast ultrasound, and dermoscopy demonstrate that MedCore can remove
a large fraction of MedSAM parameters while preserving strong region and boundary performance. In a $7 \times 7$ head-MLP pruning sweep, head-pruning steps exhibit $2.887\times$ larger 95th-percentile boundary leverage than MLP-pruning steps, with all 15 valid paired comparisons favoring MLP pruning. At the metric level, the BF1 and HD95 damage densities of head pruning are respectively $2.607\times$
and $2.432\times$ those of MLP pruning, indicating that MedSAM lies in a head-fragile boundary regime: attention heads are not merely low-parameter components, but can have high boundary leverage.

\section{Related Work}
\label{related}

\textbf{Promptable segmentation and medical SAM.}
SAM formulates segmentation as a promptable task using a ViT image encoder, a prompt encoder, and a lightweight mask decoder \citep{kirillov2023segment,dosovitskiy2021vit}. MedSAM adapts this design to medical images via large-scale image-mask pairs \citep{ma2024segment}. Subsequent work improves medical adaptation along several axes: parameter-efficient tuning (SAMed, SAM-Med2D, Medical-SAM-Adapter \citep{zhang2023samed,cheng2023sammed2d,wu2023medicalsamadapter}; Polyp-SAM for transfer to polyp segmentation \citep{li2023polypsam}), boundary-aware feature design (I-MedSAM with implicit representations \citep{wei2024imedsam}; LDFSAM with distilled feature prompting \citep{zhao2026ldfsam}), 3D/video extension (MedSAM2 \citep{ma2025medsam2}), and feature fusion with task-specific networks \citep{zhong2025medsam2featurefusion,mu2026sam2polypnet}. These methods focus on prompting, fusion, or boundary localization. We study a complementary question: how to \emph{compress} an already adapted MedSAM checkpoint while preserving the structures that became important during SAM-to-MedSAM adaptation.

\textbf{Efficient SAM and medical segmentation models.}
The heavy ViT image encoder of SAM has stimulated extensive research on efficient
segmentation, including lightweight encoder replacement, knowledge distillation,
quantization, and structured pruning
\citep{zhao2023fastsam,xiong2023efficientsam,zhou2025edgesam,chen2024slimsam,lu2025qmedsam}. Medical settings also have efficient segmentation backbones, including U-Net, nnU-Net, Swin-Unet, EMCAD, MK-UNet, and quantized medical SAM variants \citep{ronneberger2015unet,isensee2021nnunet,swinunet,rahman2024emcad,rahman2025mkunet,lu2025qmedsam}.Recent lightweight polyp models, such as HSSAM-Net, further show that careful multi-scale aggregation and boundary-aware enhancement can give strong accuracy with small parameter budgets \citep{feng2025hssamnet}. These works confirm the need for efficient segmentation. Most of them redesign the architecture, train a compact surrogate, or quantize a model. MedCore instead directly prunes MedSAM and uses both the original SAM weights $\vtheta^S$ and the adapted MedSAM weights $\vtheta^M$.

\textbf{Pruning and Transformer compression.}
Classic pruning methods estimate the loss increase caused by removing weights or groups. Optimal Brain Damage and Optimal Brain Surgeon use second-order information to approximate this increase \citep{lecun1989obd,hassibi1993obs}. Deep Compression combines pruning, quantization, and coding \citep{han2015deepcompression}. Later methods use movement, Fisher information, Hessian-aware saliency, or group-level sensitivity to prune pretrained and Transformer models \citep{sanh2020movement,liu2021groupfisher,yang2023nvit}. Attention-head pruning studies show that many heads can be redundant, but some specialized heads are important and task-dependent \citep{michel2019sixteen,voita2019analyzing}. This literature motivates structured pruning, but it usually estimates importance from one endpoint checkpoint.

\textbf{Boundary preservation in medical segmentation.}
Medical segmentation models often optimize region losses such as Dice and cross-entropy. These losses can keep coarse foreground overlap but still miss fine contours. To address this limitation, prior work has explored boundary losses, surface- or Hausdorff-inspired objectives, topology-aware regularization, and boundary-aware architectures for more accurate contour modeling \citep{kervadec2019boundary,karimi2019hausdorff,shit2021cldice,shao2024polyper,li2024bmfanet,liu2024fcanet,tong2026begaunet}. Different from these task-specific segmentation methods, our work brings boundary preservation into model compression by asking not only whether pruning preserves Dice, but also whether it perturbs boundary-critical structures and the boundary logit field.

\section{MedCore: Boundary-Preserving Medical Core Pruning}

MedCore has two design goals. The first goal is to preserve structures that became important during medical adaptation. The second goal is to preserve structures whose removal causes large boundary displacement.

\begin{figure*}[t]
    \centering
    \includegraphics[width=\textwidth]{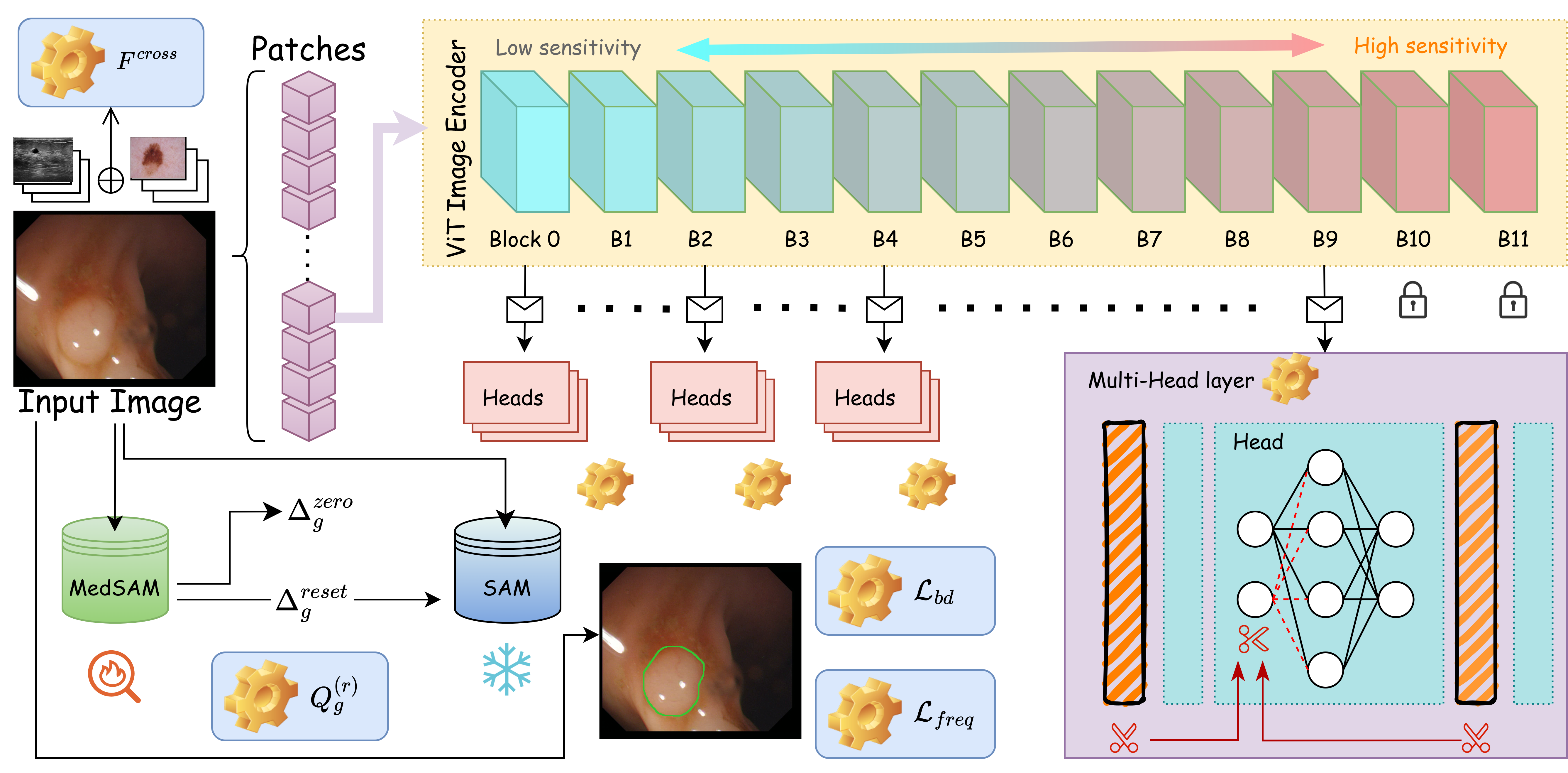}
    \caption{MedCore pruning overview. The Cross-Fisher signal $F^{\mathrm{cross}}$ 
combines SAM and MedSAM gradients to score attention heads and MLP channels 
via $Q^{(r)}_g$, driving block-sensitive cascade pruning of the ViT encoder 
(blocks B10--B11 protected) with boundary- and frequency-aware recovery. 
Yellow gears mark where pruning operations are applied: Cross-Fisher 
estimation, priority computation, per-block head pruning, and 
recovery-loss-driven stabilization.}
    \label{fig:overview}
\end{figure*}

\subsection{Problem Formulation}

A training or calibration sample is written as$\vz=(I,\vb,M)$, where $I\in \R^{3\times H\times W}$ is a medical image, $\vb\in \R^4$ is a bounding-box prompt, and $M\in\{0,1\}^{H\times W}$ is a binary mask. MedSAM maps $(I,\vb)$ to a logit map $s_{\vtheta,\vphi}(I,\vb)$ and a probability mask $\widehat M=\sigma(s_{\vtheta,\vphi}(I,\vb))$. The parameter vector $\vtheta=(\vtheta_E,\vtheta_D)$ contains the image encoder and mask decoder. The prompt encoder parameter $\vphi$ is frozen.

Given a calibration set $\cD_{\mathrm{cal}}=\{\vz_n\}_{n=1}^{N}$, we define the empirical medical risk as
\begin{equation}
\widehat{\cR}(\vtheta)
=
\frac{1}{N}\sum_{n=1}^{N}\ell_{\mathrm{seg}}(\vz_n;\vtheta),
\qquad
\ell_{\mathrm{seg}}=\ell_{\mathrm{Dice}}+\ell_{\mathrm{BCE}} .
\label{eq:empirical_risk}
\end{equation}
We partition pruneable parameters into structured groups $\cG=\{g_1,\ldots,g_K\}$. In our implementation, groups are attention heads and MLP connection groups in the ViT image encoder. A binary group mask $\vs\in\{0,1\}^K$ gives the masked parameter $\vs\odot_{\cG}\vtheta$. If $c_g$ is the cost of group $g$, the ideal pruning problem is
\begin{equation}
\min_{\vtheta,\vs}\ \widehat{\cR}(\vs\odot_{\cG}\vtheta)
\quad
\mathrm{s.t.}
\quad
\sum_{g\in\cG}s_g c_g\le B .
\label{eq:ideal_pruning}
\end{equation}
This problem is combinatorial. MedCore approximates it by scoring groups and removing low-score groups under a block-aware budget. We keep two checkpoints: $\vtheta^S$ is the original SAM checkpoint, and $\vtheta^M$ is the MedSAM checkpoint. Pruning starts from $\vtheta^M$, while $\vtheta^S$ serves as the adaptation reference.

\subsection{Boundary Leverage Principle}

We now state the geometric principle that motivates boundary-aware pruning. For a fixed image and prompt, we write the logit map as $s_{\vtheta}(\vx)$, where $\vx$ indexes a spatial location. The predicted boundary is the zero level set of this logit map. A compression operation $G$ changes the model to $\vtheta^{-G}$ and induces the logit perturbation $\delta_G(\vx)=s_{\vtheta^{-G}}(\vx)-s_{\vtheta}(\vx)$. The following assumption is standard in level-set perturbation analysis. It only says that the boundary is not locally flat in the logit field.

\begin{assumption}[Regular boundary]
\label{ass:regular_boundary}
For a given image and prompt, the logit map $s_{\vtheta}$ is continuously differentiable in a neighborhood of its zero level set $\Gamma_{\vtheta}=\{\vx:s_{\vtheta}(\vx)=0\}$. There exists $\kappa>0$ such that $\|\nabla s_{\vtheta}(\vx)\|_2\ge \kappa$ for every $\vx\in\Gamma_{\vtheta}$.
\end{assumption}

\begin{definition}[Boundary leverage]
\label{def:boundary_leverage}
Let $G$ be a compression operation with compression gain $\Delta C_G>0$. Let $\cB$ be a narrow boundary band around a ground-truth or predicted boundary. We define the boundary leverage of $G$ as
\begin{equation}
\lambda_G^{\mathrm{bd}}
=
\frac{1}{\Delta C_G}
\E_{\vx\in\cB}
\left[
\frac{|\delta_G(\vx)|}{\|\nabla s_{\vtheta}(\vx)\|_2+\eps}
\right].
\label{eq:boundary_leverage}
\end{equation}
Here $\eps>0$ is a small constant. This quantity estimates the boundary displacement caused by one unit of compression.
\end{definition}

\begin{theorem}[First-order boundary perturbation]
\label{thm:boundary_perturbation}
Under Assumption~\ref{ass:regular_boundary}, suppose that $\delta_G$ is sufficiently small in a neighborhood of $\Gamma_{\vtheta}$. For each $\vx\in\Gamma_{\vtheta}$, let $u_G(\vx)$ be the signed normal displacement from the old boundary to the new boundary induced by $G$. Then
\begin{equation}
 u_G(\vx)
=
-
\frac{\delta_G(\vx)}{\|\nabla s_{\vtheta}(\vx)\|_2}
+O(\|\delta_G\|_{C^1}^2).
\label{eq:boundary_displacement}
\end{equation}
\end{theorem}

Theorem~\ref{thm:boundary_perturbation} explains why boundary metrics can fail before Dice fails. Boundary displacement depends on the logit perturbation at the boundary and on the local logit slope. A group with small parameter count can still be dangerous if it causes a large boundary logit perturbation. We prove the theorem in Appendix~\ref{app:proofs}.

\begin{proposition}[Boundary-aware budget rule]
\label{prop:budget_rule}
Let $c_H$ and $c_M$ be compression budgets assigned to heads and MLPs, and let $\cE_{\mathrm{bd}}(c_H,c_M)$ be a locally differentiable boundary error. If
\[
\frac{\partial \cE_{\mathrm{bd}}}{\partial c_H}
>
\frac{\partial \cE_{\mathrm{bd}}}{\partial c_M}
\]
at a feasible interior allocation, then moving a small amount of compression budget from head pruning to MLP pruning decreases boundary error to first order while keeping the total compression fixed.
\end{proposition}

Proposition~\ref{prop:budget_rule} turns boundary leverage into a budget principle. It does not say that heads are always more important than MLPs. It says that the safer compression direction is the one with lower marginal boundary damage. Section~\ref{sec:boundary_leverage_analysis} verifies that, for MedSAM, head-pruning steps have much higher boundary leverage than MLP-pruning steps.

\subsection{Medical-Core Scoring via Dual Intervention}

The boundary leverage principle says which structures are dangerous for boundaries. We also need to preserve structures that became important during medical adaptation. MedSAM has a useful reference checkpoint, namely the original SAM checkpoint $\vtheta^S$. We therefore score each group $g$ by two counterfactual interventions.

The zero intervention sets group $g$ to zero while keeping all other groups unchanged. If $T_g^0$ denotes this operation, the exact zero cost is
\begin{equation}
\Delta_g^{\mathrm{zero}}
=
\widehat{\cR}(\vs\odot_{\cG}T_g^0(\vtheta^M))
-
\widehat{\cR}(\vs\odot_{\cG}\vtheta^M).
\label{eq:zero_intervention}
\end{equation}
A small value means that $g$ contributes little to current MedSAM predictions. 

The reset intervention replaces group $g$ in MedSAM by its SAM value. If $T_g^S$ denotes this operation, the exact reset cost is
\begin{equation}
\Delta_g^{\mathrm{reset}}
=
\widehat{\cR}(\vs\odot_{\cG}T_g^S(\vtheta^M))
-
\widehat{\cR}(\vs\odot_{\cG}\vtheta^M).
\label{eq:reset_intervention}
\end{equation}
A large reset cost means that the change from $\vtheta_g^S$ to $\vtheta_g^M$ is functionally important for medical segmentation. This signal is different from the zero cost. It measures adaptation-specific importance rather than only endpoint importance. We combine both costs as
\begin{equation}
Q_g
=
\alpha_{b(g)}\widehat\Delta_g^{\mathrm{zero}}
+
(1-\alpha_{b(g)})\widehat\Delta_g^{\mathrm{reset}},
\label{eq:fused_score}
\end{equation}
where $b(g)$ is the Transformer block containing $g$. We use a block-level mixture weight because different blocks respond differently to boundary and adaptation signals. The pruning priority is
\begin{equation}
P_g=\frac{Q_g}{(c_g+\eps)^{\tau}},
\label{eq:priority}
\end{equation}
where $\tau\ge 0$ controls cost normalization. Groups with smaller $P_g$ are pruned first.

\subsection{Boundary-Aware Fisher and Cross-Fisher Approximation}

Computing ~\eqref{eq:zero_intervention} and ~\eqref{eq:reset_intervention} exactly for every group is expensive. We use a local second-order approximation. Around $\vtheta^M$, zeroing group $g$ corresponds to the perturbation $-\vtheta_g^M$, and resetting group $g$ corresponds to the perturbation $\vtheta_g^S-\vtheta_g^M$. We approximate the Hessian by diagonal Fisher information.

The loss used for Fisher estimation should reflect the metric we want to preserve. We therefore use a boundary-weighted loss. Let $B(M)$ be a morphological boundary map computed by dilation minus erosion. With pixel index $u$, we define
\begin{equation}
\ell_{\mathrm{bd}}(\vz;\vtheta)
=\sum_{u}
\left(1+\lambda_{\mathrm{bd}}B(M)_u\right)
\ell_{\mathrm{BCE}}(s_{\vtheta,\vphi}(I,\vb)_u,M_u) \\
+\ell_{\mathrm{Dice}}(\widehat M,M).
\label{eq:boundary_loss}
\end{equation}
For a sub-distribution $r$, such as one dataset or modality, the MedSAM Fisher is
\begin{equation}
F^M_{r,i}
=
\frac{1}{|\cD_{\mathrm{cal}}^{(r)}|}
\sum_{\vz_n\in\cD_{\mathrm{cal}}^{(r)}}
\left(
\frac{\partial \ell_{\mathrm{bd}}(\vz_n;\vtheta^M)}{\partial \theta_i}
\right)^2 .
\label{eq:medsam_fisher}
\end{equation}
This gives the zero approximation
\[
\widehat\Delta_g^{\mathrm{zero},(r)}
=
\frac{1}{2}\sum_{i\in g}F^M_{r,i}(\theta_i^M)^2.
\]

For the reset cost, the perturbation is the SAM-to-MedSAM parameter change. We also compute the SAM Fisher $F_i^S$ on the same calibration data and use a geometric mean,
\begin{equation}
F^{\mathrm{cross}}_{r,i}
=
\sqrt{F^M_{r,i}F^S_i+\eps_F}.
\label{eq:cross_fisher}
\end{equation}
The reset approximation becomes
\begin{equation}
\widehat\Delta_g^{\mathrm{reset},(r)}
=
\frac{1}{2}\sum_{i\in g}
F^{\mathrm{cross}}_{r,i}(\theta_i^M-\theta_i^S)^2.
\label{eq:reset_approx}
\end{equation}
The geometric mean down-weights parameters that are insensitive in either SAM or MedSAM. This prevents the reset score from being dominated by large but functionally irrelevant weight shifts.

\subsection{Distribution-Aware Aggregation and Budget Allocation}

Medical data are heterogeneous. A structure that appears redundant on one dataset may be important on another dataset. We compute per-distribution scores $\widehat Q_g^{(r)}$ using ~\eqref{eq:fused_score} and aggregate them as
\begin{equation}
\widehat Q_g^{\mathrm{dist}}
=
\sum_{r=1}^{R}\pi_r\widehat Q_g^{(r)}
+\beta\,\mathrm{Var}_{r}[\widehat Q_g^{(r)}],
\label{eq:dist_score}
\end{equation}
where $\pi_r\ge 0$ and $\sum_r\pi_r=1$. The variance term makes pruning more conservative for groups that are important on only some sub-distributions.

We also allocate pruning budgets non-uniformly across Transformer blocks. Let $\cB_{\ell}$ be the parameter set of block $\ell$. We estimate block sensitivity by
\[
S_{\ell}=\sum_{i\in\cB_{\ell}}F_i,
\qquad
F_i=\sum_r\pi_rF^M_{r,i}.
\]
Blocks with high $S_{\ell}$ receive smaller pruning quotas. In the main one-time pruning experiments, the deepest blocks are protected because they are closest to the mask decoder and have the highest Fisher sensitivity. In the extreme sequential setting, we apply a second conservative pruning pass to previously protected blocks after the first checkpoint has been selected.

\subsection{Head-to-MLP Cascade and Lightweight Recovery}

MedCore prunes in a cascade. We first prune attention heads. Then we run a short recovery phase. We then prune MLP connection groups from the head-pruned model. Order ablations show a slight advantage for head-to-MLP pruning (Appendix~\ref{app:pruning_order}), supporting this default for further mechanism analysis.

During pruning, we recover the remaining weights on the calibration set with a lightweight objective:
\begin{equation}
\cL_{\mathrm{rec}}
=
\cL_{\mathrm{seg}}
+\lambda_1\cL_{\mathrm{bd}}
+\lambda_2\cL_{\mathrm{feat}}
+\lambda_3\cL_{\mathrm{logit}}
+\lambda_4\cL_{\mathrm{freq}}.
\label{eq:recovery_loss}
\end{equation}
Here $\cL_{\mathrm{bd}}$ is boundary-weighted BCE, $\cL_{\mathrm{feat}}$ distills encoder features from the unpruned MedSAM teacher, $\cL_{\mathrm{logit}}$ distills boundary-region logits, and $\cL_{\mathrm{freq}}$ penalizes high-frequency mask discrepancies. Recovery is not used to define the pruning score. It lets the remaining structures re-coordinate after structural removal.

\section{Experiments}

\subsection{Experimental Setup}
\label{sec:Experimental_Setup}
We evaluate MedCore on the official MedSAM ViT-B checkpoint, using SAM ViT-B as the reference and freezing the prompt encoder. Experiments cover polyp endoscopy (CVC-ClinicDB~\cite{BERNAL201599}, CVC-ColonDB~\cite{https://doi.org/10.1155/2017/4037190}, Kvasir-SEG~\cite{10.1007/978-3-030-37734-2_37}), breast ultrasound (BUSI~\cite{ALDHABYANI2020104863}), dermoscopy (ISIC2018~\cite{Tschandl_2018, codella2019skinlesionanalysismelanoma}), and multi-class cardiac MRI (ACDC Heart MR~\cite{ACDC_Heart_MR}). Calibration and implementation details are provided in Appendix~\ref{app:boundary_leverage_details}.

We then report Dice, Intersection over Union (IoU), Boundary F1 (BF1), 95th-percentile Hausdorff distance (HD95), parameters, and FLOPs. BF1 and HD95 measure boundary fidelity. Recovery denotes lightweight calibration-set stabilization during pruning; fine-tuning denotes post-pruning training on a target dataset. Table~\ref{tab:main_results} also includes unpruned MedSAM fine-tuned with the same protocol. The one-time component ablations in Section~\ref{sec:component_ablation} omit post-pruning fine-tuning to isolate component effects.

\subsection{Main Results on Polyp Segmentation}

Table~\ref{tab:main_results} reports macro-averaged results on the three polyp benchmarks. MedCore reaches strong compression while preserving boundary quality after fine-tuning. The h50\_m70 configuration removes 60.0\% of parameters and reduces FLOPs by 58.4\%, while achieving Dice 0.9549, BF1 0.6388, and HD95 5.14. The h70\_m95 configuration removes 72.2\% of parameters and still obtains BF1 0.6259 and HD95 6.78. The sequential configuration h84\_m95 reaches 12.1M parameters and 90.4G FLOPs, while keeping strong region and boundary metrics.

Fine-tuned MedSAM provides the full-model reference under the same fine-tuning protocol. Relative to this reference, sequential h84\_m95 has 0.0028 lower Dice, 0.0347 lower BF1, and 0.01 higher HD95. This configuration retains comparable Dice and HD95 under heavy compression, while the BF1 gap reflects the trade-off in boundary preservation.

\begin{table}[t]
\centering
\caption{Macro-averaged results on CVC-ClinicDB~\cite{BERNAL201599}, CVC-ColonDB~\cite{https://doi.org/10.1155/2017/4037190}, and Kvasir-SEG~\cite{10.1007/978-3-030-37734-2_37}. MedSAM + fine-tuning and MedCore use the same fine-tuning protocol. One-time uses one head-to-MLP cascade; Sequential adds a second pruning pass to h70\_m95. Bold marks the best values within MedCore.}
\label{tab:main_results}
\small
\setlength{\tabcolsep}{3.5pt}
\begin{tabular}{llrrrrrr}
\toprule
Category & Method & Params$\downarrow$ & FLOPs$\downarrow$ & Dice$\uparrow$ & IoU$\uparrow$ & BF1$\uparrow$ & HD95$\downarrow$ \\
\midrule
Baseline & MedSAM \citep{ma2024segment} & 89.7M & 926.5G & 0.9191 & 0.8648 & 0.5321 & 21.29 \\
 & MedSAM + fine-tuning & 89.7M & 926.5G & 0.9578 & 0.9201 & 0.6809 & 5.11 \\
\midrule
Efficient SAM & EfficientSAM-s \citep{xiong2023efficientsam} & 26.4M & 188.0G & 0.8765 & 0.8131 & 0.4340 & 29.73 \\
 & EfficientSAM-t \citep{xiong2023efficientsam} & 10.2M & 55.9G & 0.8723 & 0.8103 & 0.4218 & 25.98 \\
 & SlimSAM \citep{chen2024slimsam} & 28.0M & 189.9G & 0.8435 & 0.7758 & 0.3922 & 34.11 \\
\midrule
Medical Seg. & Swin-Unet \citep{swinunet} & 27.2M & 5.92G & 0.7845 & 0.6933 & 0.4347 & 23.88 \\
 & nnU-Net \citep{ronneberger2015unet, isensee2021nnunet} & 46.3M & 75.1G & 0.9093 & 0.8574 & 0.5382 & 32.15 \\
 & EMCAD \citep{rahman2024emcad} & 26.8M & 10.6G & 0.9141 & 0.8608 & 0.5305 & 30.31 \\
 & MK-UNet \citep{rahman2025mkunet} & 0.32M & 0.61G & 0.8658 & 0.7986 & 0.4635 & 46.21 \\
\midrule
Medical SAM & SAMed \citep{zhang2023samed} & 91.2M & 206.7G & 0.7419 & 0.6640 & 0.3283 & 97.57 \\
 & QMedSAM \citep{lu2025qmedsam} & 9.79M & 75.7G & 0.8089 & 0.7108 & 0.4118 & 14.55 \\
\midrule
\rowcolor{MedCoreBlue} MedCore One-time & h40\_m30 & 61.3M & 620.7G & 0.9547 & 0.9173 & \textbf{0.6508} & 7.23 \\
\rowcolor{MedCoreBlue} & h50\_m70 & 35.8M & 385.2G & 0.9549 & 0.9169 & 0.6388 & 5.14 \\
\rowcolor{MedCoreBlue} & h70\_m70 & 30.1M & 287.5G & 0.9524 & 0.9130 & 0.6342 & 6.23 \\
\rowcolor{MedCoreBlue} & h70\_m95 & 24.9M & 245.0G & 0.9522 & 0.9124 & 0.6259 & 6.78 \\
\midrule
\rowcolor{MedCoreGreen} MedCore Sequential & h60\_m70 & 15.6M & 132.6G & 0.9547 & \textbf{0.9176} & 0.6506 & 5.91 \\
\rowcolor{MedCoreGreen} & h72\_m84 & 13.7M & 112.4G & 0.9536 & 0.9153 & 0.6337 & 5.47 \\
\rowcolor{MedCoreGreen} & h84\_m95 & \textbf{12.1M} & \textbf{90.4G} & \textbf{0.9550} & 0.9174 & 0.6462 & \textbf{5.12} \\
\bottomrule
\end{tabular}
\end{table}

\begin{figure*}[t]
    \centering
    \includegraphics[width=\textwidth]{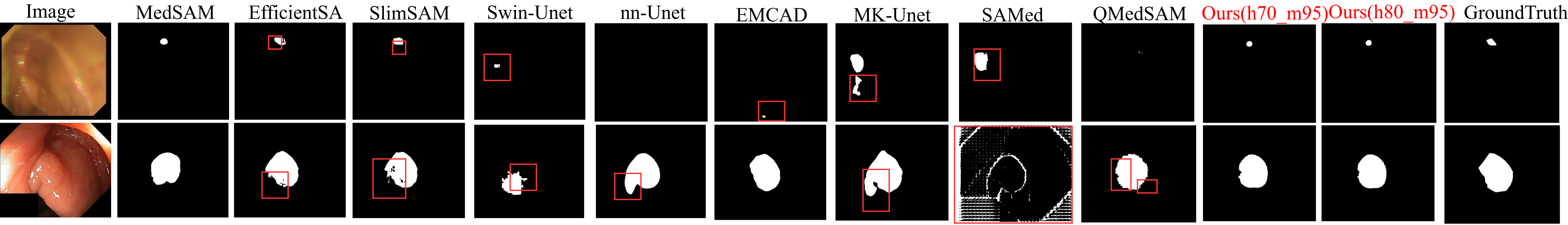}
    \caption{
    Qualitative segmentation comparison between our two high-pruning configurations and other models.
    Incorrect segmentation regions and boundary artifacts are highlighted with red boxes.
    }
    \label{fig:qualitative_results}
\end{figure*}

\subsection{Cross-Modality Generalization}

Table~\ref{tab:cross_modality} evaluates the same pruning pipeline on CVC-ClinicDB, BUSI, ISIC2018, and ACDC Heart MR. On ClinicDB and ISIC2018, MedCore improves boundary metrics after fine-tuning. On BUSI, moderate compression preserves performance closely, while aggressive compression begins to degrade BF1 and HD95 (Appendix Table~\ref{tab:cross_modality_full}). This pattern is useful because it shows both the strength and the limit of direct cross-modality transfer. MedCore is not a replacement for target-domain validation. It is a compression framework that can preserve strong performance when the calibration and fine-tuning data are representative.

\begin{table}[t]
\centering
\caption{Cross-modality results with post-pruning fine-tuning. Params/FLOPs: 89.7M/926.5G for MedSAM and 24.9M/240.9G for MedCore h70\_m95.}
\label{tab:cross_modality}
\small
\setlength{\tabcolsep}{4pt}
\resizebox{\linewidth}{!}{%
\begin{tabular}{llrrrrrrr}
\toprule
\multirow{2}{*}{Method}
 & \multirow{2}{*}{Metric}
 & \multicolumn{1}{c}{\multirow{2}{*}{CVC-ClinicDB~\cite{BERNAL201599}}}
 & \multicolumn{1}{c}{\multirow{2}{*}{BUSI~\cite{ALDHABYANI2020104863}}}
 & \multicolumn{1}{c}{\multirow{2}{*}{ISIC2018~\cite{Tschandl_2018, codella2019skinlesionanalysismelanoma}}}
 & \multicolumn{4}{c}{ACDC Heart MR~\cite{ACDC_Heart_MR}} \\
\cmidrule(lr){6-9}
 & & & & & Myocardium & LV & RV & Average \\
\midrule
\multirow{3}{*}{MedSAM}
 & Dice$\uparrow$ & 0.9264 & 0.9334 & 0.9314 & 0.7934 & 0.9235 & 0.9577 & 0.8915 \\
 & BF1$\uparrow$ & 0.5477 & 0.4989 & 0.2401 & 0.7685 & 0.8232 & 0.8805 & 0.8240 \\
 & HD95$\downarrow$ & 13.77 & 7.30 & 7.47 & 7.20 & 2.58 & 0.87 & 3.55 \\
\midrule
\multirow{3}{*}{\shortstack[l]{MedCore\\h70\_m95}}
 & Dice$\uparrow$ & 0.9523 & 0.9298 & 0.9530 & 0.8173 & 0.9069 & 0.9463 & 0.8902 \\
 & BF1$\uparrow$ & 0.6194 & 0.4703 & 0.4753 & 0.7121 & 0.7450 & 0.7949 & 0.7507 \\
 & HD95$\downarrow$ & 6.30 & 9.43 & 5.86 & 5.87 & 4.65 & 1.51 & 4.01 \\
\bottomrule
\end{tabular}%
}
\end{table}

\section{Analysis and Ablation}

\subsection{Component Ablation: What Drives Boundary Preservation?}
\label{sec:component_ablation}

Table~\ref{tab:component_ablation} ablates the main components. The one-time columns report pruning results before post-pruning fine-tuning to isolate the structural effect of pruning decisions. The sequential column uses the same post-pruning fine-tuning protocol for all variants.

Boundary-aware Fisher provides the most direct contour protection. Removing it reduces BF1 across the reported settings, with the largest reduction at the most aggressive one-time setting h70\_m95. Reset-to-SAM scoring contributes less at moderate compression and becomes important as the pruning budget becomes severe. This pattern supports its role as a safeguard for medical adaptation under aggressive compression. The distribution variance term similarly protects structures whose importance varies across datasets. Its contribution becomes more useful when fewer structures remain to support predictions across the different datasets.

\begin{table}[t]
\centering
\caption{Component ablation on macro-average BF1 across the three polyp benchmarks. One-time columns are reported without post-pruning fine-tuning. The sequential column uses the same post-pruning fine-tuning for all variants. $\Delta$ is the difference from the full method at the same compression level.}
\label{tab:component_ablation}
\small
\setlength{\tabcolsep}{3pt}
\begin{tabular}{lrrrrrrrrrr}
\toprule
 & \multicolumn{2}{c}{h40\_m30} & \multicolumn{2}{c}{h50\_m70} & \multicolumn{2}{c}{h70\_m70} & \multicolumn{2}{c}{h70\_m95} & \multicolumn{2}{c}{Seq. h84\_m95} \\
Method & BF1 & $\Delta$ & BF1 & $\Delta$ & BF1 & $\Delta$ & BF1 & $\Delta$ & BF1 & $\Delta$ \\
\midrule
Full MedCore & 0.5162 & -- & 0.4078 & -- & 0.3144 & -- & 0.4050 & -- & 0.6462 & -- \\
w/o Boundary Fisher & 0.5035 & -0.013 & 0.3724 & -0.035 & 0.3096 & -0.005 & 0.2781 & -0.127 & 0.6304 & -0.016 \\
w/o reset & 0.5149 & -0.001 & 0.3967 & -0.011 & 0.3628 & +0.048 & 0.3227 & -0.082 & 0.6279 & -0.018 \\
w/o variance & 0.5221 & +0.006 & 0.3997 & -0.008 & 0.3812 & +0.067 & 0.3249 & -0.080 & 0.6462 & 0.000 \\
\bottomrule
\end{tabular}
\end{table}

These results support a hierarchy of component roles within MedCore. Boundary Fisher provides the most direct protection for boundary-core structures, while reset-to-SAM scoring preserves medical-core structures created by adaptation. Variance aggregation complements these signals by retaining groups with uneven importance across datasets. Together, the components couple contour protection with safeguards for medical adaptation under aggressive compression.

\subsection{Boundary Leverage Reveals a Head-Fragile Regime}
\label{sec:boundary_leverage_analysis}

We test the empirical prediction in Definition~\ref{def:boundary_leverage} using a $7{\times}7$ sweep over $h\in\{0.3,\dots,0.9\}$ and $m\in\{0.3,0.5,0.7,0.8,0.85,0.9,0.95\}$, with 384 calibration images (128 each from Kvasir, ClinicDB, ColonDB). For each adjacent head step $G_H(i,j){:}(h_i,m_j){\to}(h_{i+1},m_j)$ and MLP step $G_M(i,j){:}(h_i,m_j){\to}(h_i,m_{j+1})$, we compute the 95th-percentile version of Eq.~\eqref{eq:boundary_leverage} on the boundary band, skipping steps with zero parameter reduction.

Table~\ref{tab:boundary_leverage} reports median $\lambda^{\mathrm{bd}}_{95}$ values of 3.961 for head steps and 1.372 for MLP steps, giving a head/MLP ratio of 2.887. Head steps have higher leverage in all 15 valid paired comparisons, supporting the empirical law $\mathrm{Median}(\lambda_G^{\mathrm{bd}}\mid G\in\cH) > \mathrm{Median}(\lambda_G^{\mathrm{bd}}\mid G\in\cM)$ for MedSAM.

\begin{table}[t]
\centering
\caption{Step-level boundary leverage analysis on the head-MLP pruning sweep. The logit-level boundary leverage result agrees with metric-level BF1 and HD95 damage densities.}
\label{tab:boundary_leverage}
\small
\setlength{\tabcolsep}{5pt}
\begin{tabular}{lccc}
\toprule
Quantity & Head steps & MLP steps & Head/MLP ratio \\
\midrule
Median $\lambda^{\mathrm{bd}}_{95}$ & 3.961 & 1.372 & 2.887 \\
BF1 damage density & 0.007692 & 0.002951 & 2.607 \\
HD95 damage density & 0.025406 & 0.010445 & 2.432 \\
\midrule
Valid steps & 35 & 21 & -- \\
Paired win rate & \multicolumn{3}{c}{15/15 = 1.000} \\
\bottomrule
\end{tabular}
\end{table}

The BF1 and HD95 rows confirm the connection between logit-level boundary leverage and final-mask degradation: head pruning incurs $2.4$--$2.6\times$ greater BF1 and HD95 damage per one percent parameter reduction. Head pruning also has a high full-image perturbation ratio, indicating perturbations across both the boundary band and the wider image. These combined observations identify MedSAM heads as high-boundary-leverage components, making aggressive head pruning risky for boundary-sensitive segmentation.

\subsection{Head-MLP Compression Landscape}

The boundary leverage result explains how the head/MLP allocation shapes the observed compression landscape. At approximately 60\% parameter reduction, h30\_m80 yields BF1 0.376 while h50\_m70 yields 0.408 before post-pruning fine-tuning, showing that boundary quality depends on structural allocation even at similar model sizes. Figure~\ref{fig:mlp_landscape} examines a slice at fixed $h=0.30$: BF1 initially decreases as $m$ increases, then rises once the effective parameter count saturates at the min-retention plateau. This non-monotonic behavior shows that both block allocation and structure type determine the relationship between sparsity and boundary performance.

\begin{figure}[t]
\centering
\includegraphics[width=0.95\linewidth]{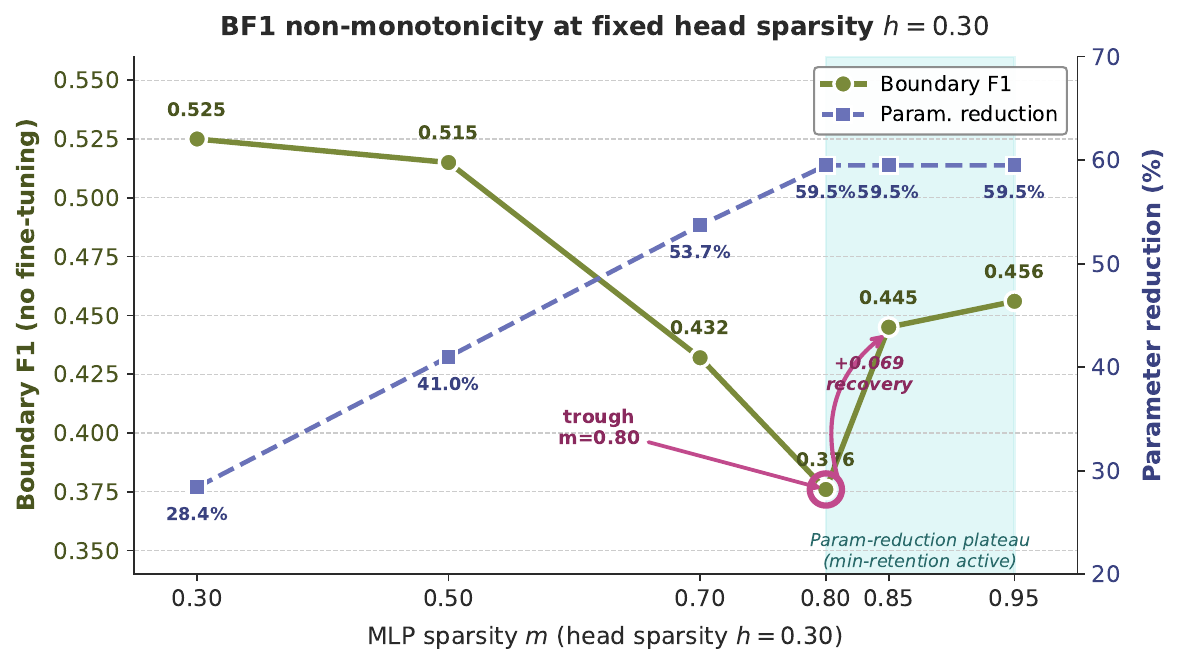}
\caption{BF1 non-monotonicity at fixed head sparsity $h=0.30$ (no post-pruning fine-tuning). 
BF1 drops to a trough of 0.376 at $m=0.80$, then recovers to 0.456 at 
$m=0.95$ along the 59.5\% parameter-reduction plateau where the 
min-retention constraint is active.}
\label{fig:mlp_landscape}
\end{figure}

Boundary-aware allocation can therefore matter as much as nominal sparsity in choosing a compression configuration. The benefit of further sparsity depends on the resulting allocation of surviving structures. Together with Table~\ref{tab:boundary_leverage}, these results support conservative head pruning and assigning most compression to lower-leverage MLP structures, with calibration analysis guiding the head/MLP balance for each target setting.

\section{Conclusion}

We introduced MedCore, a structured pruning framework that treats MedSAM pruning as medical-core extraction. A reset-to-SAM intervention preserves structures that became important during SAM-to-MedSAM adaptation, while boundary-aware Fisher estimation protects contour-critical structures. Boundary leverage provides a geometric view of compression-induced boundary damage, explaining how high Dice coexists with degraded boundaries and why attention heads require conservative pruning in our setting. Experiments demonstrate substantial parameter and FLOPs reductions while retaining strong region and boundary performance across medical segmentation datasets. MedCore inherits MedSAM's underlying capabilities and biases, so clinical deployment requires task-specific validation and prospective safety evaluation.

\clearpage

\subsection*{AI use statement}


In this work, we used generative AI tools for assisting in the writing of proofs and assisting with translation. We have not used generative AI tools for helping develop theoretical models or conceptual frameworks, formulating mathematical claims, providing critical ingredients for proving mathematical claims, proposing or refining hypotheses, designing or providing feedback on research  methodology or experiments, implementing methods, supporting qualitative and thematic data analysis, interpreting results,
and generating synthetic data sets, cleaning and reformating dataset are not applicable to this work.
Additionally, we used generative AI tools for organizing software code, brainstorming, sourcing/searching for information, and editing a research paper to improve readability. We have reviewed all AI-assisted work. We manually checked LLM-assisted proofs for correctness and consistency with the paper's core arguments. LLM-generated code was also verified and tested for correctness by all authors. We take responsibility for the final content of this work, including text, claims or artifacts produced with the aid of generative AI.

\subsection*{Ethics statement}


The research conducted in the paper conform, in every respect, with the Code of Ethics. The work uses public medical segmentation datasets and public model checkpoints for model compression research, and it does not involve deceptive data collection, unsafe data release, or human-subject interaction. We also discuss clinical validation and deployment risks in the impact statement. We have also checked and confirmed that this article does not involve potential conflicts of interest and sponsorship, discrimination/bias/fairness concerns, privacy and security issues, legal compliance, and research integrity issues (e.g., IRB, documentation, research ethics). 

\subsection*{Reproducibility statement}


The paper describes the model structure, datasets, calibration sample sizes, pruning groups, pruning schedule, recovery setting, metrics, and algorithms in Sections~\ref{sec:boundary_leverage_analysis}-~\ref{sec:component_ablation} and Appendices ~\ref{app:pruning_order}-~\ref{app:implementation_details}. These details are sufficient to reproduce the main experimental pipeline. The experiments use public datasets and public SAM/MedSAM checkpoints, and the current submission provide an anonymized code release. However, we do not provide complete reproduction scripts or private checkpoints. We plan to release code and pruned checkpoints after acceptance, together with data preparation and evaluation instructions.

Besides, for details, Section ~\ref{sec:Experimental_Setup} and Appendix ~\ref{app:implementation_details} specify the datasets, fixed train/test split protocol, calibration size, prompt setup, optimizer, learning rates, batch sizes, recovery steps, post-pruning fine-tuning, and evaluation metrics.

\bibliography{iclr2027_conference}

@inproceedings{kirillov2023segment,
  title={Segment anything},
  author={Kirillov, Alexander and Mintun, Eric and Ravi, Nikhila and Mao, Hanzi and Rolland, Chloe and Gustafson, Laura and Xiao, Tete and Whitehead, Spencer and Berg, Alexander C and Lo, Wan-Yen and others},
  booktitle={2023 IEEE/CVF international conference on computer vision (ICCV)},
  pages={3992--4003},
  year={2023},
  organization={IEEE}
}

@article{ma2024segment,
  title={Segment anything in medical images},
  author={Ma, Jun and He, Yuting and Li, Feifei and Han, Lin and You, Chenyu and Wang, Bo},
  journal={Nature communications},
  volume={15},
  number={1},
  pages={654},
  year={2024},
  publisher={Nature Publishing Group UK London}
}

@misc{dosovitskiy2021vit,
      title={An Image is Worth 16x16 Words: Transformers for Image Recognition at Scale}, 
      author={Alexey Dosovitskiy and Lucas Beyer and Alexander Kolesnikov and Dirk Weissenborn and Xiaohua Zhai and Thomas Unterthiner and Mostafa Dehghani and Matthias Minderer and Georg Heigold and Sylvain Gelly and Jakob Uszkoreit and Neil Houlsby},
      year={2021},
      eprint={2010.11929},
      archivePrefix={arXiv},
      primaryClass={cs.CV},
      url={https://arxiv.org/abs/2010.11929}, 
}

@misc{zhang2023samed,
      title={Customized Segment Anything Model for Medical Image Segmentation}, 
      author={Kaidong Zhang and Dong Liu},
      year={2023},
      eprint={2304.13785},
      archivePrefix={arXiv},
      primaryClass={cs.CV},
      url={https://arxiv.org/abs/2304.13785}, 
}

@misc{cheng2023sammed2d,
      title={SAM-Med2D}, 
      author={Junlong Cheng and Jin Ye and Zhongying Deng and Jianpin Chen and Tianbin Li and Haoyu Wang and Yanzhou Su and Ziyan Huang and Jilong Chen and Lei Jiang and Hui Sun and Junjun He and Shaoting Zhang and Min Zhu and Yu Qiao},
      year={2023},
      eprint={2308.16184},
      archivePrefix={arXiv},
      primaryClass={cs.CV},
      url={https://arxiv.org/abs/2308.16184}, 
}

@article{wu2023medicalsamadapter,
  title={Medical sam adapter: Adapting segment anything model for medical image segmentation},
  author={Wu, Junde and Wang, Ziyue and Hong, Mingxuan and Ji, Wei and Fu, Huazhu and Xu, Yanwu and Xu, Min and Jin, Yueming},
  journal={Medical image analysis},
  volume={102},
  pages={103547},
  year={2025},
  publisher={Elsevier}
}

@misc{zhao2023fastsam,
      title={Fast Segment Anything}, 
      author={Xu Zhao and Wenchao Ding and Yongqi An and Yinglong Du and Tao Yu and Min Li and Ming Tang and Jinqiao Wang},
      year={2023},
      eprint={2306.12156},
      archivePrefix={arXiv},
      primaryClass={cs.CV},
      url={https://arxiv.org/abs/2306.12156}, 
}

@inproceedings{xiong2023efficientsam,
  title={EfficientSAM: Leveraged Masked Image Pretraining for Efficient Segment Anything},
  author={Xiong, Yunyang and Varadarajan, Bala and Wu, Lemeng and Xiang, Xiaoyu and Xiao, Fanyi and Zhu, Chenchen and Dai, Xiaoliang and Wang, Dilin and Sun, Fei and Iandola, Forrest and Krishnamoorthi, Raghuraman and Chandra, Vikas},
  booktitle={2024 IEEE/CVF Conference on Computer Vision and Pattern Recognition (CVPR)},
  pages={16111--16121},
  year={2024},
  organization={IEEE},
  doi={10.1109/CVPR52733.2024.01525}}

@article{zhou2025edgesam,
  title={EdgeSAM: Prompt-in-the-loop distillation for SAM},
  author={Zhou, Chong and Li, Xiangtai and Loy, Chen Change and Dai, Bo},
  journal={International Journal of Computer Vision},
  volume={133},
  number={12},
  pages={8452--8468},
  year={2025},
  publisher={Springer}
}

@article{chen2024slimsam,
  title={SlimSAM: 0.1\% data makes segment anything slim},
  author={Chen, Zigeng and Fang, Gongfan and Ma, Xinyin and Wang, Xinchao},
  journal={Advances in Neural Information Processing Systems},
  volume={37},
  pages={39434--39461},
  year={2024}
}

@inproceedings{ronneberger2015unet,
  title={U-net: Convolutional networks for biomedical image segmentation},
  author={Ronneberger, Olaf and Fischer, Philipp and Brox, Thomas},
  booktitle={International Conference on Medical image computing and computer-assisted intervention},
  pages={234--241},
  year={2015},
  organization={Springer}
}

@article{isensee2021nnunet,
  title={nnU-Net: a self-configuring method for deep learning-based biomedical image segmentation},
  author={Isensee, Fabian and Jaeger, Paul F and Kohl, Simon AA and Petersen, Jens and Maier-Hein, Klaus H},
  journal={Nature methods},
  volume={18},
  number={2},
  pages={203--211},
  year={2021},
  publisher={Nature Publishing Group US New York}
}

@inproceedings{swinunet,
  title={Swin-unet: Unet-like pure transformer for medical image segmentation},
  author={Cao, Hu and Wang, Yueyue and Chen, Joy and Jiang, Dongsheng and Zhang, Xiaopeng and Tian, Qi and Wang, Manning},
  booktitle={European conference on computer vision},
  pages={205--218},
  year={2022},
  organization={Springer}
}

@inproceedings{rahman2024emcad,
  title={Emcad: Efficient multi-scale convolutional attention decoding for medical image segmentation},
  author={Rahman, Md Mostafijur and Munir, Mustafa and Marculescu, Radu},
  booktitle={2024 IEEE/CVF Conference on Computer Vision and Pattern Recognition (CVPR)},
  pages={11769--11779},
  year={2024},
  organization={IEEE}
}

@inproceedings{rahman2025mkunet,
  title={Mk-unet: Multi-kernel lightweight cnn for medical image segmentation},
  author={Rahman, Md Mostafijur and Marculescu, Radu},
  booktitle={2025 IEEE/CVF International Conference on Computer Vision Workshops (ICCVW)},
  pages={1053--1062},
  year={2025},
  organization={IEEE}
}

@inproceedings{lu2025qmedsam,
  title={Efficient Quantization-Aware Training on Segment Anything Model in Medical Images and Its Deployment},
  author={Lu, Haisheng and Fu, Yujie and Zhang, Fan and Zhang, Le},
  booktitle={Medical Image Segmentation Foundation Models. CVPR 2024 Challenge: Segment Anything in Medical Images on Laptop},
  pages={137--150},
  year={2025},
  publisher={Springer Nature Switzerland}
}

@inproceedings{lecun1989obd,
  title={Optimal Brain Damage},
  author={LeCun, Yann and Denker, John S. and Solla, Sara A.},
  booktitle={Advances in Neural Information Processing Systems},
  volume={2},
  year={1989}
}

@article{hassibi1993obs,
  title={Second order derivatives for network pruning: Optimal brain surgeon},
  author={Hassibi, Babak and Stork, David},
  journal={Advances in neural information processing systems},
  volume={5},
  year={1992}
}

@inproceedings{han2015deepcompression,
  title={Deep Compression: Compressing Deep Neural Networks with Pruning, Trained Quantization and Huffman Coding},
  author={Han, Song and Mao, Huizi and Dally, William J.},
  booktitle={International Conference on Learning Representations},
  year={2016}
}

@article{sanh2020movement,
  title={Movement pruning: Adaptive sparsity by fine-tuning},
  author={Sanh, Victor and Wolf, Thomas and Rush, Alexander},
  journal={Advances in neural information processing systems},
  volume={33},
  pages={20378--20389},
  year={2020}
}

@article{michel2019sixteen,
  title={Are sixteen heads really better than one?},
  author={Michel, Paul and Levy, Omer and Neubig, Graham},
  journal={Advances in neural information processing systems},
  volume={32},
  year={2019}
}

@inproceedings{voita2019analyzing,
  title={Analyzing multi-head self-attention: Specialized heads do the heavy lifting, the rest can be pruned},
  author={Voita, Elena and Talbot, David and Moiseev, Fedor and Sennrich, Rico and Titov, Ivan},
  booktitle={Proceedings of the 57th annual meeting of the association for computational linguistics},
  pages={5797--5808},
  year={2019}
}

@inproceedings{liu2021groupfisher,
  title={Group fisher pruning for practical network compression},
  author={Liu, Liyang and Zhang, Shilong and Kuang, Zhanghui and Zhou, Aojun and Xue, Jing-Hao and Wang, Xinjiang and Chen, Yimin and Yang, Wenming and Liao, Qingmin and Zhang, Wayne},
  booktitle={International conference on machine learning},
  pages={7021--7032},
  year={2021},
  organization={PMLR}
}

@inproceedings{yang2023nvit,
  title={Global vision transformer pruning with hessian-aware saliency},
  author={Yang, Huanrui and Yin, Hongxu and Shen, Maying and Molchanov, Pavlo and Li, Hai and Kautz, Jan},
  booktitle={2023 IEEE/CVF Conference on Computer Vision and Pattern Recognition (CVPR)},
  pages={18547--18557},
  year={2023},
  organization={IEEE}
}

@inproceedings{kervadec2019boundary,
  title={Boundary loss for highly unbalanced segmentation},
  author={Kervadec, Hoel and Bouchtiba, Jihene and Desrosiers, Christian and Granger, Eric and Dolz, Jose and Ayed, Ismail Ben},
  booktitle={International conference on medical imaging with deep learning},
  pages={285--296},
  year={2019},
  organization={PMLR}
}

@article{karimi2019hausdorff,
  title={Reducing the Hausdorff Distance in Medical Image Segmentation with Convolutional Neural Networks},
  author={Karimi, Davood and Salcudean, Septimiu E.},
  journal={IEEE Transactions on Medical Imaging},
  volume={39},
  number={2},
  pages={499--513},
  year={2020}
}

@inproceedings{shit2021cldice,
  title={clDice-a novel topology-preserving loss function for tubular structure segmentation},
  author={Shit, Suprosanna and Paetzold, Johannes C and Sekuboyina, Anjany and Ezhov, Ivan and Unger, Alexander and Zhylka, Andrey and Pluim, Josien PW and Bauer, Ulrich and Menze, Bjoern H},
  booktitle={2021 IEEE/CVF Conference on Computer Vision and Pattern Recognition (CVPR)},
  pages={16555--16564},
  year={2021},
  organization={IEEE}
}

@inproceedings{li2023polypsam,
  title={Polyp-{SAM}: Transfer {SAM} for polyp segmentation},
  author={Li, Yuheng and Hu, Mingzhe and Yang, Xiaofeng},
  booktitle={Medical imaging 2024: computer-aided diagnosis},
  volume={12927},
  pages={749--754},
  year={2024},
  organization={SPIE}
}

@inproceedings{wei2024imedsam,
  title={{I-MedSAM}: Implicit Medical Image Segmentation with Segment Anything},
  author={Wei, Xiaobao and Cao, Jiajun and Jin, Yizhu and Lu, Ming and Wang, Guangyu and Zhang, Shanghang},
  booktitle={European Conference on Computer Vision},
  pages={90--107},
  year={2024},
  organization={Springer},
  doi={10.1007/978-3-031-72684-2_6},
  url={https://arxiv.org/abs/2311.17081}
}

@article{zhao2026ldfsam,
  title={{LDFSAM}: Localization Distillation-Enhanced Feature Prompting SAM for Medical Image Segmentation},
  author={Zhao, Xuanbo and Wang, Cheng and Xu, Huaxing and Zhou, Hong and Yu, Zekuan and Chen, Tao and Wei, Xiaoling and Zhang, Rongjun},
  journal={Journal of Imaging},
  volume={12},
  number={2},
  pages={74},
  year={2026},
  doi={10.3390/jimaging12020074},
  url={https://www.mdpi.com/2313-433X/12/2/74}
}

@misc{ma2025medsam2,
  title={{MedSAM2}: Segment Anything in 3D Medical Images and Videos},
  author={Ma, Jun and Yang, Zongxin and Kim, Sumin and Chen, Bihui and Baharoon, Mohammed and Fallahpour, Adibvafa and Asakereh, Reza and Lyu, Hongwei and Wang, Bo},
  year={2025},
  eprint={2504.03600},
  archivePrefix={arXiv},
  primaryClass={eess.IV},
  url={https://arxiv.org/abs/2504.03600}
}

@article{zhong2025medsam2featurefusion,
  title={{MedSAM}/{MedSAM2} Feature Fusion: Enhancing {nnUNet} for 2D {TOF-MRA} Brain Vessel Segmentation},
  author={Zhong, Han and Zhang, Jiatian and Zhao, Lingxiao},
  journal={Journal of Imaging},
  volume={11},
  number={6},
  pages={202},
  year={2025},
  doi={10.3390/jimaging11060202},
  url={https://www.mdpi.com/2313-433X/11/6/202}
}

@article{mu2026sam2polypnet,
  title={{SAM2-PolypNet}: {SAM2} with Adaptive Context Enhancement Model for Polyp Segmentation},
  author={Mu, Zhaoting and Ning, Yu and Zhao, Yueqi and Zhang, Rong and Mubtasim, Fuad Md and Ning, Hailong},
  journal={Biomedical Signal Processing and Control},
  year={2026},
  note={In press},
  url={https://www.sciencedirect.com/science/article/pii/S1746809426009353}
}

@article{feng2025hssamnet,
  title={{HSSAM-Net}: Hyper-Scale Shifted Aggregation Network for Precise Colorectal Polyp Segmentation in Endoscopic Images},
  author={Feng, Qing and Ahmed, Shahzad and Zhang, Yueming and He, Lan and Yaqub, Muhammad},
  journal={Scientific Reports},
  volume={15},
  pages={38146},
  year={2025},
  doi={10.1038/s41598-025-21954-y},
  url={https://www.nature.com/articles/s41598-025-21954-y}
}

@inproceedings{shao2024polyper,
  title={Polyper: Boundary Sensitive Polyp Segmentation},
  author={Shao, Hao and Zhang, Yang and Hou, Qibin},
  booktitle={Proceedings of the AAAI Conference on Artificial Intelligence},
  volume={38},
  pages={4731--4739},
  year={2024},
  doi={10.1609/aaai.v38i5.28274},
  url={https://ojs.aaai.org/index.php/AAAI/article/view/28274}
}

@article{li2024bmfanet,
  title={{BMFA-Net}: Boundary Constraint Multi-Level Feature Aggregation Framework for Precise Polyp Segmentation},
  author={Li, Qin and Zhang, Tianchi and Mosharaf, Parvej Md and Zhang, Jing},
  journal={Applied Sciences},
  volume={14},
  number={10},
  pages={4063},
  year={2024},
  doi={10.3390/app14104063},
  url={https://www.mdpi.com/2076-3417/14/10/4063}
}

@article{liu2024fcanet,
  title={{FCA-Net}: Fully Context-Aware Feature Aggregation Network for Medical Segmentation},
  author={Liu, Dingzhou and Deng, Hongmin and Huang, Zhengwei and Fu, Jinghao},
  journal={Biomedical Signal Processing and Control},
  volume={91},
  pages={106004},
  year={2024},
  doi={10.1016/j.bspc.2024.106004},
  url={https://www.sciencedirect.com/science/article/pii/S1746809424000624}
}

@misc{tong2026begaunet,
  title={{BEGA-UNet}: Boundary-Explicit Guided Attention {U-Net} with Multi-Scale Feature Aggregation for Colonoscopic Polyp Segmentation},
  author={Tong, Tao and Zhang, Wen and Zu, Wanni},
  year={2026},
  howpublished={medRxiv preprint},
  doi={10.64898/2026.03.04.26347608},
  url={https://www.medrxiv.org/content/10.64898/2026.03.04.26347608v2}
}

@article{BERNAL201599,
title = {WM-DOVA maps for accurate polyp highlighting in colonoscopy: Validation vs. saliency maps from physicians},
journal = {Computerized Medical Imaging and Graphics},
volume = {43},
pages = {99-111},
year = {2015},
issn = {0895-6111},
doi = {https://doi.org/10.1016/j.compmedimag.2015.02.007},
url = {https://www.sciencedirect.com/science/article/pii/S0895611115000567},
author = {Jorge Bernal and F. Javier Sánchez and Gloria Fernández-Esparrach and Debora Gil and Cristina Rodríguez and Fernando Vilariño}
}

@inproceedings{10.1007/978-3-030-37734-2_37,
  title={Kvasir-seg: A segmented polyp dataset},
  author={Jha, Debesh and Smedsrud, Pia H and Riegler, Michael A and Halvorsen, P{\aa}l and De Lange, Thomas and Johansen, Dag and Johansen, H{\aa}vard D},
  booktitle={International conference on multimedia modeling},
  pages={451--462},
  year={2019},
  organization={Springer}
}

@article{https://doi.org/10.1155/2017/4037190,
author = {Vázquez, David and Bernal, Jorge and Sánchez, F. Javier and Fernández-Esparrach, Gloria and López, Antonio M. and Romero, Adriana and Drozdzal, Michal and Courville, Aaron},
title = {A Benchmark for Endoluminal Scene Segmentation of Colonoscopy Images},
journal = {Journal of Healthcare Engineering},
volume = {2017},
number = {1},
pages = {4037190},
doi = {https://doi.org/10.1155/2017/4037190},
url = {https://onlinelibrary.wiley.com/doi/abs/10.1155/2017/4037190},
eprint = {https://onlinelibrary.wiley.com/doi/pdf/10.1155/2017/4037190},
year = {2017}
}

@article{ALDHABYANI2020104863,
title = {Dataset of breast ultrasound images},
journal = {Data in Brief},
volume = {28},
pages = {104863},
year = {2020},
issn = {2352-3409},
doi = {https://doi.org/10.1016/j.dib.2019.104863},
url = {https://www.sciencedirect.com/science/article/pii/S2352340919312181},
author = {Walid Al-Dhabyani and Mohammed Gomaa and Hussien Khaled and Aly Fahmy}
}

@article{Tschandl_2018,
   title={The HAM10000 dataset, a large collection of multi-source dermatoscopic images of common pigmented skin lesions},
   volume={5},
   ISSN={2052-4463},
   url={http://dx.doi.org/10.1038/sdata.2018.161},
   DOI={10.1038/sdata.2018.161},
   number={1},
   journal={Scientific Data},
   publisher={Springer Science and Business Media LLC},
   author={Tschandl, Philipp and Rosendahl, Cliff and Kittler, Harald},
   year={2018},
   month=Aug }

@misc{codella2019skinlesionanalysismelanoma,
      title={Skin Lesion Analysis Toward Melanoma Detection 2018: A Challenge Hosted by the International Skin Imaging Collaboration (ISIC)}, 
      author={Noel Codella and Veronica Rotemberg and Philipp Tschandl and M. Emre Celebi and Stephen Dusza and David Gutman and Brian Helba and Aadi Kalloo and Konstantinos Liopyris and Michael Marchetti and Harald Kittler and Allan Halpern},
      year={2019},
      eprint={1902.03368},
      archivePrefix={arXiv},
      primaryClass={cs.CV},
      url={https://arxiv.org/abs/1902.03368}, 
}

@misc{li2022revisitingrandomchannelpruning,
      title={Revisiting Random Channel Pruning for Neural Network Compression}, 
      author={Yawei Li and Kamil Adamczewski and Wen Li and Shuhang Gu and Radu Timofte and Luc Van Gool},
      year={2022},
      eprint={2205.05676},
      archivePrefix={arXiv},
      primaryClass={cs.CV},
      url={https://arxiv.org/abs/2205.05676}, 
}

@misc{han2015learningweightsconnectionsefficient,
      title={Learning both Weights and Connections for Efficient Neural Networks}, 
      author={Song Han and Jeff Pool and John Tran and William J. Dally},
      year={2015},
      eprint={1506.02626},
      archivePrefix={arXiv},
      primaryClass={cs.NE},
      url={https://arxiv.org/abs/1506.02626}, 
}

@misc{theis2018fastergazepredictiondense,
      title={Faster gaze prediction with dense networks and Fisher pruning}, 
      author={Lucas Theis and Iryna Korshunova and Alykhan Tejani and Ferenc Huszár},
      year={2018},
      eprint={1801.05787},
      archivePrefix={arXiv},
      primaryClass={cs.CV},
      url={https://arxiv.org/abs/1801.05787}, 
}

@ARTICLE{ACDC_Heart_MR,
  author={Bernard, Olivier and Lalande, Alain and Zotti, Clement and Cervenansky, Frederick and Yang, Xin and Heng, Pheng-Ann and Cetin, Irem and Lekadir, Karim and Camara, Oscar and Gonzalez Ballester, Miguel Angel and Sanroma, Gerard and Napel, Sandy and Petersen, Steffen and Tziritas, Georgios and Grinias, Elias and Khened, Mahendra and Kollerathu, Varghese Alex and Krishnamurthi, Ganapathy and Rohé, Marc-Michel and Pennec, Xavier and Sermesant, Maxime and Isensee, Fabian and Jäger, Paul and Maier-Hein, Klaus H. and Full, Peter M. and Wolf, Ivo and Engelhardt, Sandy and Baumgartner, Christian F. and Koch, Lisa M. and Wolterink, Jelmer M. and Išgum, Ivana and Jang, Yeonggul and Hong, Yoonmi and Patravali, Jay and Jain, Shubham and Humbert, Olivier and Jodoin, Pierre-Marc},
  journal={IEEE Transactions on Medical Imaging}, 
  title={Deep Learning Techniques for Automatic MRI Cardiac Multi-Structures Segmentation and Diagnosis: Is the Problem Solved?}, 
  year={2018},
  volume={37},
  number={11},
  pages={2514-2525},
  doi={10.1109/TMI.2018.2837502}}
\bibliographystyle{iclr2027_conference}

\appendix

\clearpage
\section*{Appendix}
\addcontentsline{toc}{section}{Appendix}
The appendix contains proofs, implementation details, and additional information about the analysis. We keep the main text focused on the method and key empirical findings.

\section{Proofs}
\label{app:proofs}

\subsection{Proof of Theorem~\ref{thm:boundary_perturbation}}

\begin{proof}
Fix a point $\vx\in\Gamma_{\vtheta}$. By definition of the old boundary, $s_{\vtheta}(\vx)=0$. Since Assumption~\ref{ass:regular_boundary} gives $\|\nabla s_{\vtheta}(\vx)\|_2>0$, the unit normal direction of the level set at $\vx$ is
\[
\vn(\vx)=\frac{\nabla s_{\vtheta}(\vx)}{\|\nabla s_{\vtheta}(\vx)\|_2}.
\]
The new logit map after compression is $\widetilde s(\vx)=s_{\vtheta}(\vx)+\delta_G(\vx)$. We seek the new boundary point on the normal line through $\vx$. We write it as
\[
\widetilde \vx=\vx+u_G(\vx)\vn(\vx),
\]
where $u_G(\vx)$ is the signed normal displacement. Because $\widetilde \vx$ lies on the new boundary, it satisfies
\[
0=\widetilde s(\widetilde \vx)=s_{\vtheta}(\widetilde \vx)+\delta_G(\widetilde \vx).
\]
We expand both terms around $\vx$. The first-order Taylor expansion gives
\[
 s_{\vtheta}(\widetilde \vx)
 =s_{\vtheta}(\vx)
 +u_G(\vx)\nabla s_{\vtheta}(\vx)^{\top}\vn(\vx)
 +O(u_G(\vx)^2).
\]
Since $s_{\vtheta}(\vx)=0$ and $\nabla s_{\vtheta}(\vx)^{\top}\vn(\vx)=\|\nabla s_{\vtheta}(\vx)\|_2$, this becomes
\[
 s_{\vtheta}(\widetilde \vx)
 =u_G(\vx)\|\nabla s_{\vtheta}(\vx)\|_2
 +O(u_G(\vx)^2).
\]
The perturbation term satisfies
\[
\delta_G(\widetilde \vx)
=\delta_G(\vx)+O(|u_G(\vx)|\|\delta_G\|_{C^1}).
\]
Substituting the two expansions into the boundary equation gives
\[
0=
 u_G(\vx)\|\nabla s_{\vtheta}(\vx)\|_2
 +\delta_G(\vx)
 +O(u_G(\vx)^2)
 +O(|u_G(\vx)|\|\delta_G\|_{C^1}).
\]
For small perturbations, the implicit function theorem guarantees that $u_G(\vx)$ is also small and has the same order as $\delta_G(\vx)$. Therefore the two remainder terms are second order in $\|\delta_G\|_{C^1}$. Solving the first-order part gives
\[
 u_G(\vx)
 =-
 \frac{\delta_G(\vx)}{\|\nabla s_{\vtheta}(\vx)\|_2}
 +O(\|\delta_G\|_{C^1}^2).
\]
This is Eq.~\eqref{eq:boundary_displacement}.
\end{proof}

\subsection{Proof of Proposition~\ref{prop:budget_rule}}

\begin{proof}
Let the current allocation be $(c_H,c_M)$ and keep the total compression fixed at $C=c_H+c_M$. We move a small amount $\eta>0$ of compression budget from heads to MLPs. The new allocation is
\[
(c_H-\eta,c_M+\eta).
\]
The first-order Taylor expansion of $\cE_{\mathrm{bd}}$ at $(c_H,c_M)$ gives
\[
\cE_{\mathrm{bd}}(c_H-\eta,c_M+\eta)
=
\cE_{\mathrm{bd}}(c_H,c_M)
-
\eta\frac{\partial \cE_{\mathrm{bd}}}{\partial c_H}
+
\eta\frac{\partial \cE_{\mathrm{bd}}}{\partial c_M}
+o(\eta).
\]
The first-order change in boundary error is therefore
\[
\Delta \cE_{\mathrm{bd}}
=
\eta
\left(
\frac{\partial \cE_{\mathrm{bd}}}{\partial c_M}
-
\frac{\partial \cE_{\mathrm{bd}}}{\partial c_H}
\right)
+o(\eta).
\]
If $\frac{\partial \cE_{\mathrm{bd}}}{\partial c_H}$ is larger than $\frac{\partial \cE_{\mathrm{bd}}}{\partial c_M}$, then the coefficient of $\eta$ is negative. For sufficiently small positive $\eta$, the first-order term dominates the $o(\eta)$ remainder, so $\Delta \cE_{\mathrm{bd}}<0$. The boundary error decreases while the total compression remains $C$.
\end{proof}

\section{Head--MLP Pruning Order and Interactions}
\label{app:pruning_order}

We compare head-to-MLP, reversed MLP-to-head, and simultaneous pruning on CVC-ClinicDB~\cite{BERNAL201599}. The light setting uses lightweight recovery on the calibration set during pruning, without post-pruning fine-tuning. The full setting additionally includes post-pruning fine-tuning.

\begin{table}[h]
\centering
\caption{Pruning-order ablation on CVC-ClinicDB. Each cell reports light/full results, without/with post-pruning fine-tuning. Best results in each setting are in bold.}
\label{tab:pruning_order_app}
\small
\setlength{\tabcolsep}{4pt}
\begin{tabular}{lcccc}
\toprule
Order & Dice$\uparrow$ & IoU$\uparrow$ & BF1$\uparrow$ & HD95$\downarrow$ \\
\midrule
Head$\rightarrow$MLP (default)
 & \textbf{0.8646 / 0.9080}
 & \textbf{0.7788 / 0.8461}
 & \textbf{0.2861 / 0.4991}
 & \textbf{28.19 / 17.39} \\
MLP$\rightarrow$Head (reversed)
 & 0.8629 / 0.9078
 & 0.7770 / 0.8457
 & 0.2757 / 0.4912
 & 28.70 / 17.71 \\
Simultaneous
 & 0.8608 / 0.9069
 & 0.7733 / 0.8452
 & 0.2792 / 0.4934
 & 28.83 / 18.14 \\
\bottomrule
\end{tabular}
\end{table}

Table~\ref{tab:pruning_order_app} shows that head-to-MLP pruning is numerically best on all four metrics in both settings. Its advantage over the reversed order is modest, while its advantage over simultaneous pruning is consistent. These results support head-to-MLP as a practical default for mechanism analysis and show limited sensitivity to order reversal in this experiment. They do not establish theoretical optimality or statistical significance.

The cascade also accounts for directional head--MLP interactions. After head pruning and the fixed lightweight recovery, we recompute the MLP Fisher and dual-intervention scores on the resulting checkpoint. MLP groups are therefore not scored independently on the intact model: their ranking is conditioned on the selected head mask and the updated surviving parameters. This captures conditional head-to-MLP interactions rather than exhaustively searching all joint head--MLP masks. The simultaneous variant omits this conditional re-scoring and performs slightly worse than the default on every reported metric.

\section{Algorithmic Summary}
\label{app:algorithm}

Algorithm~\ref{alg:medcore} summarizes MedCore. The algorithm separates the scoring stage from the recovery stage. The reset-to-SAM operation is used only to compute importance. The final compressed model removes groups or keeps them active. It does not reset kept groups to SAM weights.

\begin{algorithm}[h]
\caption{MedCore pruning}
\label{alg:medcore}
\begin{algorithmic}[1]
\Require MedSAM checkpoint $\vtheta^M$, SAM checkpoint $\vtheta^S$, frozen prompt encoder $\vphi$, calibration sets $\{\cD_{\mathrm{cal}}^{(r)}\}_{r=1}^{R}$, structured groups $\cG$, target budget $B$
\Ensure Pruned and recovered MedSAM checkpoint
\State Initialize all group masks as active.
\For{each sub-distribution $r$}
    \State Compute boundary-aware MedSAM Fisher $F^M_{r,i}$ using Eq.~\eqref{eq:medsam_fisher}.
\EndFor
\State Compute SAM Fisher $F^S_i$ on the same calibration data.
\For{each group $g\in\cG$}
    \State Estimate $\widehat\Delta_g^{\mathrm{zero},(r)}$ using $F^M_{r,i}$.
    \State Estimate $\widehat\Delta_g^{\mathrm{reset},(r)}$ using Eq.~\eqref{eq:reset_approx}.
    \State Compute the fused score $\widehat Q_g^{(r)}$ and the distribution-aware score $\widehat Q_g^{\mathrm{dist}}$.
    \State Compute the cost-normalized priority $P_g$ using Eq.~\eqref{eq:priority}.
\EndFor
\State Allocate head-pruning budgets across blocks with block sensitivity.
\State Prune attention heads with the lowest priorities under the block budgets.
\State Run a short recovery fine-tuning stage.
\State Recompute or reuse priorities for MLP connection groups in the head-pruned model.
\State Prune MLP connection groups until the target budget is reached.
\State Run post-pruning recovery fine-tuning with Eq.~\eqref{eq:recovery_loss}.
\State Physically remove pruned structures and return the compact checkpoint.
\end{algorithmic}
\end{algorithm}

\section{Additional Experimental and Boundary Leverage Details}
\label{app:boundary_leverage_details}

\subsection{Additional Experimental Settings}

Training and test splits use a fixed random seed. Polyp Fisher calibration uses 128 images in total: 64, 38, and 26 from Kvasir-SEG, CVC-ColonDB, and CVC-ClinicDB, respectively, with $\pi_r=(0.5,0.3,0.2)$ and batch size 1. We set $\lambda_{\mathrm{bd}}=3.0$, $\beta=0.3$, $\tau=0$, and $\varepsilon=10^{-8}$, keeping these hyperparameters fixed for BUSI and ISIC2018.

Unless stated otherwise, MedCore uses 100-step lightweight recovery on the calibration set during pruning at a learning rate of $10^{-5}$, with feature, boundary, logit, and frequency loss weights of $0.5$, $1.0$, $2.0$, and $0.5$, respectively. This recovery is distinct from post-pruning fine-tuning on a target dataset. ACDC results include per-structure metrics and their macro-average.

\subsection{Boundary Leverage Evaluation}

We compute boundary leverage on 384 calibration images, with 128 images from each polyp dataset. The sweep has seven head sparsity levels and seven MLP sparsity levels. We compute adjacent head steps and adjacent MLP steps. Steps with zero additional parameter reduction are undefined and are omitted.

For a head step, we fix the MLP sparsity and increase head sparsity. For an MLP step, we fix head sparsity and increase MLP sparsity. We use the same parameter-reduction unit for both families. This normalization is important because a 10\% increase in head sparsity and a 10\% increase in MLP sparsity do not remove the same number of parameters.

The boundary band width is tested at $3$ and $5$ pixels. Both choices give the same qualitative conclusion in our experiments. The main text reports the 95th-percentile boundary leverage because HD95 is also a high-percentile boundary metric. The mean boundary leverage gives the same direction but is less aligned with HD95.

The summary values are: BLR$_{95}=2.887$, RLR $=3.064$, BSR $=0.942$, paired win rate $=1.000$, and median paired difference $=2.603$. BSR is the ratio between boundary leverage ratio and region leverage ratio. Since BSR is slightly below one, we do not claim that head pruning is boundary-specific. We claim that head pruning is high-boundary-leverage. This is the statement needed to explain boundary fragility.

\section{Implementation Details}
\label{app:implementation_details}

We use the official MedSAM ViT-B checkpoint as $\vtheta^M$ and the SAM ViT-B checkpoint as $\vtheta^S$. The prompt encoder remains frozen. Fisher estimation uses calibration images and the same bounding-box prompt protocol as MedSAM. We compute boundary maps by morphological dilation minus erosion of the binary mask. We normalize losses so that the boundary-weighted BCE term does not dominate the Dice term at the beginning of recovery.

For attention-head pruning, each group contains the corresponding Q, K, V, and output-projection parameters associated with one head. For MLP pruning, each group contains a connection group in the feed-forward block. We use physical removal after mask selection so that parameter count and FLOPs reflect the actual compact model.

The main one-time configurations protect the deepest encoder blocks because they have high Fisher sensitivity and connect directly to the mask decoder. The sequential extreme setting is different. It starts from the h70\_m95 checkpoint and then performs a conservative second pruning pass on the previously protected deepest blocks. This setting is used only to test the limit of compression.

\section{Full Cross-Modality Results}
\label{app:cross_modality}

Table~\ref{tab:cross_modality_full} reports results for all available pruning configurations on CVC-ClinicDB, BUSI, and ISIC2018, complementing the representative h70\_m95 comparison in Table~\ref{tab:cross_modality}. All MedCore configurations include post-pruning fine-tuning.

\begin{table}[t]
\centering
\caption{Full cross-modality results with post-pruning fine-tuning.}
\label{tab:cross_modality_full}
\small
\setlength{\tabcolsep}{4pt}
\resizebox{\linewidth}{!}{%
\begin{tabular}{llrrrrr}
\toprule
Dataset & Method & Params$\downarrow$ & FLOPs$\downarrow$ & Dice$\uparrow$ & BF1$\uparrow$ & HD95$\downarrow$ \\
\midrule
CVC-ClinicDB~\cite{BERNAL201599}
 & MedSAM & 89.7M & 926.5G & 0.9264 & 0.5477 & 13.77 \\
 & MedCore h50\_m70 & 35.8M & 377.0G & 0.9601 & 0.6603 & 2.64 \\
 & MedCore h70\_m70 & 30.1M & 283.4G & 0.9564 & 0.6409 & 3.85 \\
 & MedCore h70\_m95 & 24.9M & 240.9G & 0.9523 & 0.6194 & 6.30 \\
 & MedCore h80\_m95 & 23.2M & 224.6G & 0.9507 & 0.5951 & 4.46 \\
\midrule
BUSI~\cite{ALDHABYANI2020104863}
 & MedSAM & 89.7M & 926.5G & 0.9334 & 0.4989 & 7.30 \\
 & MedCore h50\_m70 & 35.8M & 377.0G & 0.9357 & 0.5048 & 7.35 \\
 & MedCore h70\_m70 & 30.1M & 283.4G & 0.9336 & 0.4913 & 7.82 \\
 & MedCore h70\_m95 & 24.9M & 240.9G & 0.9298 & 0.4703 & 9.43 \\
 & MedCore h80\_m95 & 23.2M & 224.6G & 0.9316 & 0.4833 & 8.58 \\
\midrule
ISIC2018~\cite{Tschandl_2018, codella2019skinlesionanalysismelanoma}
 & MedSAM & 89.7M & 926.5G & 0.9314 & 0.2401 & 7.47 \\
 & MedCore h50\_m70 & 35.8M & 377.0G & 0.9525 & 0.4806 & 6.12 \\
 & MedCore h70\_m70 & 30.1M & 283.4G & 0.9525 & 0.4732 & 5.57 \\
 & MedCore h70\_m95 & 24.9M & 240.9G & 0.9530 & 0.4753 & 5.86 \\
 & MedCore h80\_m95 & 23.2M & 224.6G & 0.9526 & 0.4708 & 5.94 \\
\bottomrule
\end{tabular}%
}
\end{table}

On CVC-ClinicDB and ISIC2018, all reported MedCore configurations improve Dice, BF1, and HD95 relative to the MedSAM baseline. On BUSI, h50\_m70 preserves performance closely, while the more aggressive configurations have lower BF1 and higher HD95 than the baseline. These results show that the compression--performance trade-off depends on the target dataset and support the need for target-domain validation.

\section{Comparison with Structured Pruning Baselines}
\label{sec:appendix_pruning_baselines}

\begin{table}[t]
\centering
\caption{Comparison with structured pruning baselines on macro-averaged polyp segmentation metrics. All methods use MedSAM ViT-B with the same 100-step inter-stage stabilization and no terminal post-pruning fine-tuning. Parameter counts and FLOPs are identical across methods within each compression level. Best results within each compression level are in bold.}
\label{tab:pruning_baselines_app}
\footnotesize
\setlength{\tabcolsep}{4pt}
\renewcommand{\arraystretch}{1.2}
\begin{tabular}{llcccc}
\toprule
Compression & Method & Dice$\uparrow$ & IoU$\uparrow$ & BF1$\uparrow$ & HD95$\downarrow$ \\
\midrule
\multirow{9}{*}{Low: h40\_m30}
 & Random Pruning & 0.9076 & 0.8436 & 0.4598 & 23.12 \\
 & Magnitude Pruning & 0.8278 & 0.7228 & 0.2081 & 41.63 \\
 & Zero-only Fisher & 0.9137 & 0.8541 & 0.5035 & 22.13 \\
 & Vanilla Fisher & 0.9113 & 0.8512 & 0.4972 & 23.55 \\
 & Original Medical Core & 0.9123 & 0.8519 & 0.4988 & 23.50 \\
 & Movement Pruning & 0.8751 & 0.7943 & 0.3343 & 30.77 \\
 & Group Fisher & 0.9141 & 0.8547 & 0.5044 & 22.38 \\
 & Hessian-aware ViT & 0.9144 & 0.8550 & 0.5041 & 22.14 \\
 & \textbf{MedCore (ours)} & \textbf{0.9191} & \textbf{0.8615} & \textbf{0.5162} & \textbf{20.30} \\
\midrule
\multirow{9}{*}{Mid: h50\_m70}
 & Random Pruning & 0.8394 & 0.7391 & 0.2272 & 40.16 \\
 & Magnitude Pruning & 0.8237 & 0.7135 & 0.1665 & 43.60 \\
 & Zero-only Fisher & 0.8881 & 0.8124 & 0.3911 & 29.10 \\
 & Vanilla Fisher & 0.8854 & 0.8080 & 0.3859 & 29.44 \\
 & Original Medical Core & 0.8870 & 0.8106 & 0.3879 & 28.69 \\
 & Movement Pruning & 0.8450 & 0.7475 & 0.2442 & 37.18 \\
 & Group Fisher & 0.8866 & 0.8101 & 0.3888 & 29.48 \\
 & Hessian-aware ViT & 0.8868 & 0.8103 & 0.3890 & 29.46 \\
 & \textbf{MedCore (ours)} & \textbf{0.8908} & \textbf{0.8179} & \textbf{0.4078} & \textbf{27.07} \\
\midrule
\multirow{9}{*}{High: h70\_m70}
 & Random Pruning & 0.8371 & 0.7344 & 0.2121 & 41.59 \\
 & Magnitude Pruning & 0.8200 & 0.7101 & 0.1761 & 44.74 \\
 & Zero-only Fisher & 0.8555 & 0.7630 & 0.2784 & 39.26 \\
 & Vanilla Fisher & 0.8555 & 0.7634 & 0.2803 & 38.95 \\
 & Original Medical Core & 0.8558 & 0.7629 & 0.2799 & 36.25 \\
 & Movement Pruning & 0.8273 & 0.7191 & 0.1837 & 41.09 \\
 & Group Fisher & 0.8556 & 0.7637 & 0.2794 & 38.90 \\
 & Hessian-aware ViT & 0.8555 & 0.7635 & 0.2794 & 38.92 \\
 & \textbf{MedCore (ours)} & \textbf{0.8738} & \textbf{0.7899} & \textbf{0.3144} & \textbf{31.35} \\
\midrule
\multirow{9}{*}{Extreme: h84\_m95}
 & Random Pruning & 0.8180 & 0.7046 & 0.1468 & 43.95 \\
 & Magnitude Pruning & 0.8161 & 0.7033 & 0.1643 & 45.93 \\
 & Zero-only Fisher & 0.8312 & 0.7258 & 0.2003 & 42.18 \\
 & Vanilla Fisher & 0.8338 & 0.7298 & 0.2048 & 42.39 \\
 & Original Medical Core & 0.8346 & 0.7312 & 0.2046 & 42.46 \\
 & Movement Pruning & 0.8343 & 0.7302 & 0.1946 & 42.51 \\
 & Group Fisher & 0.8348 & 0.7314 & 0.2072 & 42.38 \\
 & Hessian-aware ViT & 0.8349 & 0.7314 & 0.2054 & \textbf{41.84} \\
 & \textbf{MedCore (ours)} & \textbf{0.8492} & \textbf{0.7542} & \textbf{0.2539} & 41.92 \\
\bottomrule
\end{tabular}
\end{table}

The main paper compares MedCore against external efficient SAM variants and medical segmentation backbones (Table~\ref{tab:main_results}), whose architectures and training pipelines differ. This appendix provides a controlled comparison of the \emph{scoring criterion}. All matched baselines operate on the same MedSAM ViT-B encoder, share identical structured groups (attention heads and MLP connection groups), calibration data, and evaluation metrics, and use the same target sparsity levels as MedCore. All methods receive the same light 100-step inter-stage stabilization, but \textbf{terminal post-pruning fine-tuning is disabled}, including for MedCore. The results therefore compare importance scores under matched stabilization rather than after full recovery.

\paragraph{Baselines.} We compare against eight structured pruning criteria:
(i) \textbf{Random Pruning}~\cite{li2022revisitingrandomchannelpruning}, which selects groups uniformly at random within each block; (ii) \textbf{Magnitude Pruning}~\cite{han2015learningweightsconnectionsefficient,han2015deepcompression}, which ranks groups by their parameter $\ell_2$ norm; (iii) \textbf{Zero-only Fisher}~\cite{michel2019sixteen}, which uses only $\Delta^{\mathrm{zero}}$ from Eq.~(5) and discards the reset-to-SAM signal; (iv) \textbf{Vanilla Fisher}~\cite{theis2018fastergazepredictiondense}, which uses the standard Dice+BCE loss without boundary weighting, Cross-Fisher, or distribution-aware aggregation; (v) \textbf{Original Medical Core}, which retains dual zero/reset intervention but omits boundary-aware Fisher, Cross-Fisher, multi-dataset aggregation, and block-sensitivity allocation; (vi) \textbf{Movement Pruning}~\cite{sanh2020movement}, which uses first-order movement scores adapted to the shared structured groups; (vii) \textbf{Group Fisher}~\cite{liu2021groupfisher}, which estimates group importance through Fisher information; and (viii) \textbf{Hessian-aware ViT}~\cite{yang2023nvit}, which uses Hessian-aware structural saliency.

The \textbf{head-importance heuristic}~\cite{michel2019sixteen} is reported separately because it prunes attention heads only and defines no MLP-neuron importance. We retain this head-only setting without introducing an auxiliary MLP criterion. Its smaller actual parameter reductions do not constitute an equal-total-budget comparison with the joint head--MLP methods.

\paragraph{Results.} Table~\ref{tab:pruning_baselines_app} reports macro-averaged metrics on the three polyp benchmarks across four compression levels. Table~\ref{tab:head_importance_app} reports the head-only results with their actual parameter reductions.

\begin{table}[t]
\centering
\caption{Head-importance heuristic on macro-averaged polyp segmentation metrics. Only attention heads are pruned; actual parameter reductions are relative to the unpruned MedSAM.}
\label{tab:head_importance_app}
\small
\setlength{\tabcolsep}{6pt}
\renewcommand{\arraystretch}{1.05}
\begin{tabular}{lccccc}
\toprule
Head budget & Param$\downarrow$ & Dice$\uparrow$ & IoU$\uparrow$ & BF1$\uparrow$ & HD95$\downarrow$ \\
\midrule
h40 & 12.7\% & 0.9130 & 0.8534 & 0.5031 & 23.31 \\
h50 & 15.8\% & 0.9068 & 0.8430 & 0.4752 & 24.72 \\
h70 & 22.2\% & 0.8797 & 0.8005 & 0.3640 & 30.49 \\
h84 & 24.1\% & 0.8594 & 0.7700 & 0.2930 & 34.95 \\
\bottomrule
\end{tabular}
\end{table}

\paragraph{Discussion.} Three observations follow from Table~\ref{tab:pruning_baselines_app}, consistent with the conceptual claims in the main paper.

First, Magnitude Pruning severely degrades boundary metrics across all compression levels: BF1 falls to 0.2081 even at h40\_m30, and HD95 exceeds 41. This is consistent with Theorem~3.3, which relates boundary displacement to the boundary logit perturbation and local logit slope rather than weight magnitude alone. A small-magnitude attention head can carry substantial boundary leverage that a norm-based criterion does not directly capture.

Second, Zero-only Fisher and Vanilla Fisher improve Dice and BF1 over Random and Magnitude Pruning at every level, supporting data-dependent sensitivity estimation. Group Fisher and Hessian-aware ViT provide additional competitive comparisons, but MedCore achieves the best Dice, IoU, and BF1 at all four levels and the best HD95 at three, leading in 15 of 16 metric--budget cells. At the sole exception, h84\_m95, Hessian-aware ViT obtains HD95 41.84 versus 41.92 for MedCore, while MedCore improves BF1 by 0.0485. These results are consistent with the component ablation in Section~5.1, where boundary-aware Fisher, reset-to-SAM scoring, and distribution-aware aggregation contribute to boundary preservation.

Third, MedCore's absolute BF1 margin over the strongest competing baseline increases monotonically from 0.0118 at h40\_m30 to 0.0467 at h84\_m95, corresponding to relative gains of 2.3\% and 22.5\%. This widening supports the boundary leverage principle as a design lens. As compression increases, identifying boundary-critical structures becomes more consequential, and generic sensitivity scores leave a larger gap in contour preservation. The advantage therefore extends beyond the internal Fisher variants to the additional structured-pruning criteria.

These comparisons complement Table~\ref{tab:main_results}, which evaluates architecturally distinct efficient and medical segmentation models. They address a narrower question: whether boundary-aware, adaptation-aware scoring improves pruning decisions on the same MedSAM encoder under matched calibration and stabilization. Table~\ref{tab:pruning_baselines_app} supports this conclusion, with the largest BF1 advantage at the highest compression level.

\section{Limitations}
\label{app:limitations}

MedCore uses calibration data to estimate Fisher information and boundary leverage. If this calibration set is small or biased, the score may under-protect rare structures. The reset-to-SAM score also assumes that the SAM and MedSAM checkpoints are aligned parameter by parameter. This is true for the standard MedSAM adaptation path, but it may not hold for methods that change the architecture.

The boundary leverage theorem is local. It explains first-order boundary displacement under small logit perturbations. Very aggressive pruning can create non-local effects, such as disconnected masks or missing components, that are not fully captured by the first-order formula. We therefore view boundary leverage as a useful diagnostic and design principle, not as a complete replacement for final BF1 and HD95 evaluation.

Finally, our main experiments focus on 2D prompt-driven segmentation with bounding boxes. Extension to 3D medical segmentation and other prompt types requires additional validation. The same concepts can be applied, but the boundary band, distance metric, and computational cost must be adjusted for 3D volumes.

\end{document}